\documentclass[twocolumn]{article}
\usepackage{latexsym,amssymb,amsmath,my,graphicx}

\newtheorem{theorem}{Theorem}[section]
\newtheorem{lemma}{Lemma}[section]

\def\tr{^\top}
\def\Re{\mathbb R}
\def\one{\mathbf 1}

\title{\rule[.5ex]{\textwidth}{.5pt}
Fast Solutions to Projective Monotone \\ Linear Complementarity
Problems \rule[.8ex]{\textwidth}{.5pt}}
\author{Geoffrey J. Gordon\footnote{This work was first presented
  at the 2010 NIPS workshop ``Learning and Planning from Batch Time
  Series Data.''} \\ Machine Learning Department \\ Carnegie Mellon University
  \\ {\tt ggordon@cs.cmu.edu}}

\begin{document}

\maketitle

\begin{abstract}
  We present a new interior-point potential-reduction algorithm for
  solving monotone linear complementarity problems (LCPs) that have a
  particular special structure: their matrix $M\in\Re^{n\times n}$ can
  be decomposed as $M=\Phi U + \Pi_0$, where the rank of $\Phi$
  is $k<n$, and $\Pi_0$ denotes Euclidean projection onto the
  nullspace of $\Phi\tr$.  We call such LCPs
  \emph{projective}. Our algorithm solves a monotone projective LCP to
  relative accuracy $\epsilon$ in $O(\sqrt n \ln(1/\epsilon))$
  iterations, with each iteration requiring $O(nk^2)$ flops.  This
  complexity compares favorably with interior-point algorithms for
  general monotone LCPs: these algorithms also require $O(\sqrt n
  \ln(1/\epsilon))$ iterations, but each iteration needs to solve an
  $n\times n$ system of linear equations, a much higher cost than our
  algorithm when $k\ll n$.
  Our algorithm works even though the
  solution to a projective LCP is not restricted to lie in any low-rank
  subspace.
\end{abstract}

\section{Linear complementarity problems}

The LCP for a matrix $M\in\Re^{n\times n}$ and a vector $q\in\Re^n$
is to find vectors $x,y\in\Re^n$ with
\begin{equation}
\label{eq:lcp}
x\geq 0 \quad y\geq 0 \quad
y = Mx + q \quad
x\tr y = 0
\end{equation}
We say that vectors $x,y$ are \emph{feasible} if they satisfy the
first three conditions of~(\ref{eq:lcp}) (i.e., leaving off
complementarity), and we call them a \emph{solution} if they satisfy
all four conditions.  The complementarity gap $x\tr y$ is nonnegative
for any feasible point $(x,y)$, and measures how close a feasible
point is to being a solution.  (See~\cite{cottle-pang-stone} for an
overview of LCPs.)

If $M$ is positive semidefinite (but not necessarily symmetric), the
LCP is \emph{monotone}, and there exist interior-point algorithms
that solve it to relative accuracy $\epsilon$ in $O(\sqrt n
\ln(1/\epsilon))$ Newton-like iterations.  In each iteration, the main
work is to solve an $n\times n$ system of linear equations.

Suppose the matrix $M$ can be decomposed as $M=\Phi U + \Pi_0$, where
$\Phi\in\Re^{n\times k}$ and $U\in\Re^{k\times n}$ have rank $k<n$,
and $\Pi_0$ projects onto the nullspace of $\Phi\tr$ (that is,
$\Pi_0=I-\Phi\Phi^\dag$, where $^\dag$ denotes the Moore-Penrose
pseudoinverse).  In this case we call $(M,q)$ a \emph{projective} LCP
of rank $k$.  Our new algorithm solves a projective LCP in $O(\sqrt n
\ln(1/\epsilon))$ iterations, the same as for the general monotone
case, but with each iteration requiring only $O(nk^2)$ flops.

This result is an analog of the situation for linear equations: a
rank-$k$ factored system of linear equations can be solved in
$O(nk^2)$ flops, while it is believed that a general $n\times n$
system of equations requires $\Omega(n^{2+\eta})$ flops for some
constant $\eta>0$.  However, unlike the situation for linear
equations, in a projective LCP we can't \emph{a priori} restrict
either $x$ or $y$ to a low-rank subspace of $\Re^n$; so, it is perhaps
surprising that the analogous complexity result still holds.  (The
inequality constraints $x\geq 0$ and $y\geq 0$ are the source of this
difficulty: the intersection of $x\geq 0$ or $y\geq 0$ with a rank-$k$
subspace can be quite restrictive.)

\section{Potential reduction}

We say that $x$ and $y$ are \emph{strictly feasible} if they satisfy
$x>0$, $y>0$, and $y=Mx+q$.  We will assume that we know a strictly
feasible initial point $(x_0,y_0)$ for our LCP\@.  (If we do not, it
is possible to construct one, as mentioned
in~\cite{kojima-potential-reduction}.)

For any strictly feasible point $(x,y)$, fixing a parameter $\kappa >
0$, we define the potential
\begin{align}
\label{eq:pot}
p_\kappa(x,y) = (n+\kappa)\ln x\tr y - \sum_{i=1}^n \ln x_i - \sum_{i=1}^n \ln y_i
\end{align}
We will design an algorithm that attempts to reduce the potential
$p_\kappa(x,y)$ over time.  The following lemma justifies this idea:
\begin{lemma}
\label{lem:pot-bounds-gap}
For any strictly feasible $x,y$,
\begin{align}
p_0(x,y) = n \ln x\tr y - \sum_{i=1}^n \ln x_i - \sum_{i=1}^n \ln y_i > 0
\end{align}
\end{lemma}
(For a proof, see the appendix.)  In particular,
Lemma~\ref{lem:pot-bounds-gap} implies $p_\kappa(x,y)\geq \kappa\ln x\tr y$,
or $x\tr y \leq \exp(p_\kappa(x,y)/\kappa)$; so, if we can
reduce the potential by at least some amount $\delta>0$ per iteration,
after $T$ iterations our iterate $(x,y)$ will satisfy
\begin{align*}
x\tr y &\leq \exp((p_\kappa(x_0,y_0)-T\delta)/\kappa)
\end{align*}
That is, our algorithm will converge linearly: a bound on the gap will
decrease by a factor of $\exp(-\delta/\kappa)$ per iteration.  Below,
we will take $\kappa=\sqrt n$, and $\delta$ will not depend on $n$;
so, if we desire a reduction of our potential by a factor
$0<\epsilon<1$, we will need $\frac{\sqrt n}{\delta} \ln(1/\epsilon)$
iterations, as the abstract states.

\section{The central path}

Lemma~\ref{lem:pot-bounds-gap} shows that $p_0(x,y)> 0$.
The local minimizers of $p_0$ are the points where its
gradient vanishes:
\begin{align*}
0 = \textstyle \frac{d}{dx_i} p_0(x,y) &= ny_i/x\tr y - 1/x_i\\
0 = \textstyle \frac{d}{dy_i} p_0(x,y) &= nx_i/x\tr y - 1/y_i
\end{align*}
Multiplying the first equation through by $x_i(x\tr y)/n$, or the second
equation through by $y_i(x\tr y)/n$, we get
\begin{align*}
x_iy_i &= x\tr y/n
\end{align*}
which is satisfied for a pair $x,y>0$ if and only if $x_iy_i = t$ for
all $i$ and some $t>0$.  Equivalently, we can write $x\circ y =
t\one$, where $\circ$ denotes the Hadamard (elementwise) product.

The points $(x,y)$ that satisfy
\begin{align*}
x> 0 \quad y > 0 \quad
y = Mx+q \quad x\circ y = t\one
\end{align*}
are called the \emph{central path} of the LCP $(M,q)$; for monotone
LCPs, if the central path is nonempty, it is a smooth curve, and it
approaches the solution of the LCP as $t\to 0$.  We can view the term
$p_0(x,y)$ as encouraging our algorithm to remain close to the central
path; the remaining part of our potential, $\kappa\ln(x\tr y)$,
encourages our algorithm to slide along the central path, reducing
$t=x\tr y/n$ and pushing us closer to a solution.

\section{A result on rank}
\label{sec:rank}

To make our algorithm run in time $O(nk^2)$ per iteration, we will
need to do most of our calculations on vectors of length $k$ instead
of length $n$.  Unfortunately, as mentioned earlier, the vectors $x$
and $y$ are not guaranteed to lie in any rank-$k$ subspace.  Our main
insight is that we can work mostly from a function of $x$ and $y$
that does lie in a rank-$k$ subspace.  In more detail:
\begin{lemma}
  \label{lem:inrange}
  Suppose the pair $x,y$ is feasible for the monotone LCP $(M,q)$, and
  that $M=\Phi U + \Pi_0$, where $\Phi$ has rank $k$ and $\Pi_0$
  projects onto the nullspace of $\Phi\tr$.  Then the vector $x-y+q$
  is in the range of $\Phi$.
\end{lemma}
\begin{proof}
  Define $\Pi = I - \Pi_0$, so that the range of $\Pi$ is the same as
  the range of $\Phi$.  Since $\Pi_0$ projects onto the nullspace
  of $\Phi\tr$, we know $\Pi_0\Phi=0$. And, since $\Pi_0$ is a
  projection matrix, we have $\Pi_0^2=\Pi_0$.  So, $\Pi_0 M =
  \Pi_0(\Phi U+\Pi_0) = \Pi_0$.  Therefore, for any feasible $x,y$:
  \begin{align*}
    y &= Mx + q\\
    \Pi_0 y &= \Pi_0 M x + \Pi_0 q\\
    \Pi_0 y &= \Pi_0 x + \Pi_0 q\\
    (I-\Pi) y &= (I-\Pi) x + (I-\Pi) q\\
    \Pi (x-y+q) &= x-y+q
  \end{align*}
  So, $x-y+q$ is in the range of $\Pi$, as claimed.
\end{proof}

In general, we can't recover $x$ or $y$ individually from $x-y+q$.
However, if we know that the pair $x,y$ solves the LCP, we can use
complementarity to recover $x$ and $y$: $x\tr y = 0$, so for any $i$,
at most one of $x_i$ and $y_i$ can be nonzero.  
So, given $x-y+q$, we subtract $q$ to get $z=x-y$.  Then we set
$x=z_+$ and $y=z_-$, i.e., $x_i=\max(z_i,0)$ and $y_i=\max(-z_i,0)$.

At intermediate points in our algorithm, we maintain $x$ and $y$
separately, and constrain $x-y+q=\Phi w$.  We calculate the update for
$w$ first by manipulating length-$k$ vectors, and then use this result
to derive updates for $x$ and $y$ with work that is only linear in
$n$.

\section{The algorithm}
 
We will base our algorithm on a potential-reduction method due to
Kojima et al.~\cite{kojima-potential-reduction}.  The algorithm
(Fig.~\ref{fig:uip}) uses Newton's method to step toward a point on
the central path with $x\circ y=t\one$.  It has a single parameter
$\beta\in(0,1)$ that helps us choose the target complementarity $t$:
we set $t$ to be a fraction $\beta$ of the current average
complementarity, i.e., $t=\beta\frac{x\tr y}{n}$.  Taking $\beta$ near
1 causes us to follow the central path closely, reducing $t$ slowly;
taking $\beta$ near 0 tries aggressively to reduce the complementarity
gap, but risks straying farther from the central path.\footnote{For
  simplicity of notation we require $x\geq 0$ and include only
  inequality constraints, but the algorithm works with minor changes
  if we allow some components of $x$ to be free instead of
  nonnegative, in which case the corresponding constraints are
  equalities instead of inequalities (i.e., the corresponding elements
  of $y$ must be zero).}

We will derive the algorithm first for general $M$ (in which case it
is the same as that of Kojima et al.), and then specialize it to the
projective case $M=\Phi U + \Pi_0$.  We will assume that $M$ is
positive semidefinite, i.e., that $\Phi U$ is positive semidefinite.

\begin{figure}
\begin{flushleft}
\rule[1ex]{\columnwidth}{0.5pt}
  In: LCP ($M,q$); strictly feasible $x,y$; $\beta\in(0,1)$; $\epsilon > 0$.\\
  Out: strictly feasible $x, y$ with $x\tr y\leq\epsilon$.
\end{flushleft}
\begin{centering}
\begin{enumerate}
  \setlength{\itemsep}{.5ex}
  \setlength{\parskip}{0pt}
\item Stop if $x\tr y\leq\epsilon$.
\item Solve~(\ref{eq:update-direction}--\ref{eq:g-def}) for an update
  $(\Delta x,\Delta y)$.
\item Choose a step length $\theta\geq 0$ by~(\ref{eq:step-length}).
\item Set $(x,y) \leftarrow (x,y) + \theta (\Delta x, \Delta y)$.
\item Repeat from step 1.
\end{enumerate}
\rule[2ex]{\columnwidth}{0.5pt}
\vspace{-3ex}
\end{centering}
\caption{Potential-reduction algorithm.}
\label{fig:uip}
\end{figure}

To step toward a point on the central path with complementarity $t$,
we want to satisfy both $x\circ y = t\one$ and $y=Mx+q$.  The
first-order Taylor approximation to $x\circ y = t\one$ is
\[
x\circ y + \Delta x \circ y + x \circ \Delta y = t\one
\]
Writing $X=\text{diag}(x)$ and $Y=\text{diag}(y)$, the Newton step is
therefore
\begin{align}
\label{eq:update-direction}
\left(\begin{array}{cc} Y & X\\ -M & I\end{array}\right)
\left(\begin{array}{c} \Delta x\\ \Delta y\end{array}\right) =
\left(\begin{array}{c} g \\ r\end{array}\right)
\end{align}
where
\begin{align}
  \label{eq:g-def}
  \begin{aligned}
    g &= \textstyle \beta\frac{x\tr y}{n}\one - x\circ y\\
    r &= Mx+q-y
  \end{aligned}
\end{align}
Here we have used $t=\beta\frac{x\tr y}{n}$.  Note also that $r=0$ for
any feasible $x,y$.

To pick a step length, we first define a vector $s$ and diagonal
matrix $S$ with 
\begin{align}
  \label{eq:s-def}
  s_i = S_{ii} = \sqrt{x_iy_i}
\end{align}
Write $s_0$ for the smallest element of $s$.  Like $g$, the vector
$s/s_0$ measures how far we are from the central path: if we are near
the central path, $s/s_0 \approx \one$, and we can afford to take a
relatively larger step, while if we are far from the central path,
some elements of $s/s_0$ will be large, and we will need to be more
cautious.

In particular, we will show below that the step size
\begin{align}
  \label{eq:step-length}
  \theta = \textstyle\frac{3}{7}\frac{s_0}{\|S^{-1}g\|}
\end{align}
guarantees that we maintain strict feasibility and decrease our
potential.  Note that $s>0$ for any strictly feasible $x,y$; and,
since $\beta < 1$, we have $g\neq 0$ for any strictly feasible $x,y$.
So,~(\ref{eq:step-length}) always yields a well-defined step length
$\theta$.
In practice, the value of $\theta$ from~(\ref{eq:step-length}) will
be conservative; but, it could serve as an initializer for a line
search (e.g.,~\cite[Alg.~9.2]{boyd-vandenberghe}) to determine a step
length that decreases the potential as much as possible.

\section{Proof of correctness}
\label{sec:proof}

We proceed to show that the potential-reduction algorithm behaves as
claimed above.  The proof follows Kojima et
al.~\cite{kojima-potential-reduction}, although our presentation is
somewhat different.

\begin{lemma}
  \label{lem:unique-step}
  Eq.~\ref{eq:update-direction} defines a unique step direction.
\end{lemma}
For a proof, see the appendix.

\begin{theorem}
  \label{thm:itworks}
  Suppose $n\geq 2$, and take the parameter $\kappa$ from our
  potential~(\ref{eq:pot}) to be $\kappa=\sqrt n$.  Suppose $x$ and
  $y$ are strictly feasible for the monotone LCP $(M,q)$.  Then the
  step $(\theta\Delta x, \theta\Delta y)$
  from~(\ref{eq:update-direction}--\ref{eq:step-length}), with $\beta
  = \frac{n}{n+\kappa}$, maintains strict feasibility and guarantees a
  potential reduction of at least:
  \begin{align*}
    \textstyle
    p_\kappa(x+\theta\Delta x, y+\theta\Delta y) \leq p_\kappa(x,y) - \frac{1}{5}
  \end{align*}
\end{theorem}
\begin{proofof}{Thm.~\ref{thm:itworks}}
  To make notation simpler, we will change variables to $u =
  X^{-1}\Delta x$ and $v = Y^{-1}\Delta y$.  With this notation, and
  using~(\ref{eq:s-def}), the first row of~(\ref{eq:update-direction})
  becomes
  \begin{align}
    \label{eq:wuvg}
    S^2(u+v) = g
  \end{align}
  Our goal is now to bound the change in potential
  \[
  \Delta p = p_\kappa(x+\theta\Delta x,y+\theta\Delta y) - p_\kappa(x,y)
  \]
  We start by splitting $p_\kappa(x,y)$ into
  two pieces, so that we can bound each piece separately:
  \begin{align*}
    p_\kappa(x,y) &= \bar p_1(x,y)+ \bar p_2(x,y)\\
    \bar p_1(x,y) &= \frac{n}{\beta}\ln x\tr y\\
    \bar p_2(x,y) &= - \sum_{i=1}^n \ln x_i - \sum_{i=1}^n \ln y_i
  \end{align*}
  Here we have used $n+\kappa=\frac{n}{\beta}$, from the assumed value
  of $\beta$.
  Note that $\bar p_2$ is convex, but $\bar p_1$ is not (due to the concave
  $\ln$ function and the interaction $x\tr y$).  Write
  \begin{align*}
    \Delta\bar p_1 &= \bar p_1(x+\theta\Delta x, y+\theta\Delta y) - \bar p_1(x,y)\\
    \Delta\bar p_2 &= \bar p_2(x+\theta\Delta x, y+\theta\Delta y) - \bar p_2(x,y)
  \end{align*}
  To upper bound $\Delta\bar p_1$ we will use the identity 
  \[ \ln(z+\Delta z)\leq \ln z + \Delta z/z\]
  which holds since $\ln(z)$ is concave in $z$.
  We take $z=x\tr y$, so
  \begin{align*}
    \Delta z &= (x+\theta\Delta x)\tr (y+\theta\Delta y)-x\tr y\\
    &= \theta\Delta x\tr y + \theta x\tr\Delta y + \theta^2\Delta
    x\tr\Delta y
  \end{align*}
  Therefore,
  \begin{align*}
    \textstyle
    \Delta \bar p_1 \leq \frac{n}{\beta x\tr
      y}[\theta\Delta x\tr y + \theta x\tr\Delta y + \theta^2\Delta
    x\tr\Delta y] 
  \end{align*}
  or, in terms of $u$ and $v$,
  \begin{align*}
    \textstyle
    \Delta \bar p_1 \leq \frac{n}{\beta x\tr
      y}[\theta (x\circ y)\tr(u+v) + \theta^2\gamma] 
  \end{align*}
  where we have written $\gamma=\Delta x\tr\Delta y=u\tr XY v$.  Note
  that $\gamma\geq 0$, since $\gamma=\Delta x\tr\Delta y = \Delta x\tr
  M \Delta x$ and $M$ is positive semidefinite.

  For $\Delta\bar p_2$ we use a local upper bound on $-\ln(z+\Delta
  z)$, derived from the second-order Taylor approximation
  \begin{align*}
    -\ln(z+\Delta z) &
    \textstyle
    \approx -\ln z - \Delta z/z + \frac{1}{2}(\Delta z)^2/z^2
  \end{align*}
  To get a bound valid for some range of $\Delta z$, we scale up the
  second derivative by a factor $\tau\geq 1$:
  \begin{lemma}
    \label{lem:ln-upper}
    For $\tau \geq 1$, if $\frac{\Delta z}{z}\geq \frac{1-\tau}{\tau}$, then
    \begin{align}
      \label{eq:ln-upper}
      -\ln(z+\Delta z) &
      \textstyle
      \leq -\ln z - \Delta z/z + \frac{\tau}{2}(\Delta z)^2/z^2
    \end{align}
  \end{lemma}
  (See the appendix for a proof.)  So, using Lemma~\ref{lem:ln-upper}
  $2n$ times (first with $z=x_i$ and $\Delta z = \theta\Delta x_i$,
  and then with $z=y_i$ and $\Delta z = \theta\Delta y_i$), we have
  \begin{align*}
    \Delta\bar p_2
    \leq&\ - \theta\sum_{i=1}^n \Delta x_i/x_i  - \theta\sum_{i=1}^n
    \Delta y_i/y_i\\ 
    &{}+\theta^2\frac{\tau}{2}\sum_{i=1}^n(\Delta
    x_i)^2/x_i^2+\theta^2\frac{\tau}{2}\sum_{i=1}^n(\Delta y_i)^2/y_i^2 
  \end{align*}
  so long as $\theta\Delta x_i/x_i \geq \frac{1-\tau}{\tau}$ and
  $\theta\Delta y_i/y_i \geq \frac{1-\tau}{\tau}$ for all $i$.  Or, in
  terms of $u$ and $v$,
  \begin{align*}
    \Delta\bar p_2
    &\leq \textstyle - \theta\one\tr (u+v) + \frac{\tau}{2}\theta^2
    (u\tr u + v\tr v) 
  \end{align*}
  as long as 
  \begin{align}
    \label{eq:tau-constr}
    \textstyle
    \theta u\geq \frac{1-\tau}{\tau}\one \qquad \theta v\geq
    \frac{1-\tau}{\tau}\one
  \end{align}
  Combining the bounds on $\Delta\bar p_1$ and $\Delta\bar p_2$, we
  have
  \begin{align}
    \label{eq:pot-diff}
    \Delta p
    &\leq \textstyle 
    \frac{n}{\beta x\tr y} 
    [\theta (x\circ y)\tr(u+v) + \theta^2\gamma]\\
    \nonumber
    &\quad \textstyle {}- \theta\one\tr (u+v)
    +\frac{\tau}{2}\theta^2(u\tr u + v\tr v)
  \end{align}
  as long as~(\ref{eq:tau-constr}) holds.
  Using the definition~(\ref{eq:g-def}) of $g$, we can split the
  right-hand side of~(\ref{eq:pot-diff}) into a term that is linear in
  $\theta$:
  \[
  \textstyle
  -\frac{n}{\beta x\tr y}\theta g\tr(u+v)
  \]
  and a term that is quadratic in $\theta$:
  \[
  \textstyle
  \theta^2[\frac{n}{\beta x\tr y}\gamma +
  \frac{\tau}{2}(u\tr u+v\tr v)]
  \]
  We can simplify each of these terms separately:
  using~(\ref{eq:wuvg}), we have
  \[
  g\tr(u+v)=g\tr S^{-2}g = \|S^{-1}g\|^2
  \]
  And, 
  \begin{align*}
    u\tr u + v\tr v &= \|S^{-1}Su\|^2 + \|S^{-1}Sv\|^2\\
    &\leq \textstyle\frac{1}{s_0^2}(\|Su\|^2 + \|Sv\|^2)\\
    &=\textstyle\frac{1}{s_0^2}(\|S(u+v)\|^2-2\gamma)\\
    &= \textstyle\frac{1}{s_0^2}(\|S^{-1}g\|^2-2\gamma)
  \end{align*}
  (The second line holds by definition of $s_0$; the third
  uses the definition of $\gamma$; and the last uses~(\ref{eq:wuvg})
  again.)

  So,~(\ref{eq:pot-diff}) becomes
  \begin{align}
    \label{eq:pot-diff2}
    \Delta p &\leq \textstyle
    \theta^2[(\frac{n}{\beta x\tr y}-\frac{\tau}{s_0^2})\gamma + 
    \frac{\tau}{2s_0^2}\|S^{-1}g\|^2]\\
    &\quad \textstyle {}-\theta\frac{n}{\beta x\tr y}\|S^{-1}g\|^2 \nonumber
  \end{align}
  as long as~(\ref{eq:tau-constr}) holds.  In Lemma~\ref{lem:ln-upper}, we
  are free to choose $\tau\geq 1$; so, we will assume
  \begin{align}
    \label{eq:tau-constr2}
    \tau \geq \frac{1}{\beta}
  \end{align}
  So, since
  \[
  \textstyle \frac{x\tr y}{n} = \frac{s\tr s}{n} \geq s_0^2
  \]
  we have that $(\frac{n}{\beta x\tr y}-\frac{\tau}{s_0^2})\gamma\leq
  0$, and~(\ref{eq:pot-diff2}) becomes
  \begin{align}
    \label{eq:pot-diff3}
    \Delta p &\leq \textstyle
    \theta^2
    \frac{\tau}{2s_0^2}\|S^{-1}g\|^2 
    -\theta\frac{n}{\beta x\tr y}\|S^{-1}g\|^2 
  \end{align}
  as long as~(\ref{eq:tau-constr}) and~(\ref{eq:tau-constr2}) hold.  The
  right-hand side will be negative for the optimal $\theta>0$, since
  its derivative with respect to $\theta$ is negative at $\theta=0$.
  So, we now know that $\Delta p < 0$, i.e., $(\Delta x,\Delta y)$ is
  a descent direction for $p$ as desired.

  To determine how large a decrease in potential we can achieve, we
  need to pick a feasible step size $\theta$.  To ensure that we
  satisfy~(\ref{eq:tau-constr}), we will enforce the stricter
  constraints
  \begin{align}
    \label{eq:step-constr}
    \textstyle
    \|\theta u\|_\infty\leq\frac{\tau-1}{\tau}\qquad
    \|\theta v\|_\infty\leq\frac{\tau-1}{\tau}
  \end{align}
  Note that~(\ref{eq:step-constr}) implies that our step maintains
  strict feasibility: since $\frac{\tau-1}{\tau}<1$, we have $\|\theta
  u\|_\infty = \|\theta X^{-1}\Delta x\|_\infty < 1$ and $\|\theta
  v\|_\infty = \|\theta Y^{-1}\Delta y\|_\infty< 1$.  Now,
  \begin{align*}
    \|u\|_\infty &\leq \|u\|\\
    &= \|S^{-1}Su\|\\
    &\leq \textstyle\frac{1}{s_0} \|Su\|\\
    &\leq \textstyle\frac{1}{s_0} (\|Su\|+\|Sv\|+2\gamma)\\
    &= \textstyle\frac{1}{s_0} \|Su+Sv\|\\
    &= \textstyle\frac{1}{s_0} \|S^{-1}g\|
  \end{align*}
  Analogously,
  \begin{align*}
    \|v\|_\infty &\leq \textstyle\frac{1}{s_0} \|S^{-1}g\|
  \end{align*}
  So,~(\ref{eq:step-constr}) will be satisfied if we take
  \begin{align}
    \label{eq:theta}
    \theta &= \textstyle\frac{\tau-1}{\tau}\frac{s_0}{\|S^{-1}g\|}
  \end{align}
  Substituting into~(\ref{eq:pot-diff3}), we have
  \begin{align}
    \Delta p &\leq 
    \textstyle\frac{(\tau-1)^2}{\tau^2}\frac{s_0^2}{\|S^{-1}g\|^2}\frac{\tau}{2s_0^2}\|S^{-1}g\|^2
    \nonumber \\
    & \textstyle \quad {} -\frac{\tau-1}{\tau}\frac{s_0}{\|S^{-1}g\|}\frac{n}{\beta
      x\tr y}\|S^{-1}g\|^2\nonumber \\
    &= \textstyle\frac{1}{2}\frac{(\tau-1)^2}{\tau}
    -\frac{\tau-1}{\tau}\frac{ns_0}{\beta s\tr
      s}\|S^{-1}g\| \label{eq:pot-diff4}
  \end{align}
  as long as~(\ref{eq:tau-constr2}) holds.
  Finally, we lower-bound $\|S^{-1}g\|$ with the following lemma,
  whose proof is in the appendix:
  \begin{lemma}
    \label{lem:glen}
    For $g$ in~(\ref{eq:g-def}), if $\beta=\frac{n}{n+\sqrt n}$, then:
    \[\textstyle \|S^{-1}g\| \geq \frac{\sqrt{3}}{2}\frac{s\tr
      s}{n}\frac{\beta}{s_0} \]
  \end{lemma}
  Substituting Lemma~\ref{lem:glen} into~(\ref{eq:pot-diff4}), we have
  \begin{align*}
    \Delta p &\leq \textstyle\frac{1}{2}\frac{(\tau-1)^2}{\tau}
    -\frac{\tau-1}{\tau}\frac{\sqrt{3}}{2}
  \end{align*}
  In particular, if we take $\tau=\frac{7}{4}$, we
  satisfy~(\ref{eq:tau-constr2}): $n\geq 2$, so $\frac{1}{\beta}\leq
  \frac{2+\sqrt{2}}{2} \approx 1.707$.  And, we have
  \begin{align*}
    \Delta p &\leq \textstyle\frac{1}{2}\frac{9}{16}\frac{4}{7}
    -\frac{3}{7}\frac{\sqrt{3}}{2} \leq -\frac{1}{5}
  \end{align*}
  as claimed.  This value of $\tau$, together with~(\ref{eq:theta}),
  yields~(\ref{eq:step-length}).
\end{proofof}

\section{Algorithm for projective LCPs}

The main work in each iteration of the potential-reduction algorithm
is to compute the Newton direction~(\ref{eq:update-direction}), which
requires solving an $n\times n$ system of linear equations.  (The
system~(\ref{eq:update-direction}) as a whole is $2n\times 2n$, but we
can use the sparsity of the three diagonal blocks to eliminate cheaply
down to an $n\times n$ system.)  The work required to solve this
$n\times n$ system can vary greatly, depending on the structure of
$M$, but is often prohibitive for large $n$.

So, in the projective case ($M=\Phi U+\Pi_0$
and $\Pi_0=I-\Phi\Phi^\dag$, where $\Phi$ has $k<n$ columns), we want
to avoid solving an $n\times n$ system at all; instead we will
construct and solve only a smaller $k\times k$ system.  Constructing
the $k\times k$ system will then be the main work in each iteration,
at $O(nk^2)$ flops.  (Solving the $k\times k$ system takes at most
$O(k^3)$ flops even if we just use simple Gaussian elimination.)

To run our potential-reduction algorithm on a projective LCP, our
basic idea (as discussed in Sec.~\ref{sec:rank}) is to keep track of
$w$ such that $x-y+q=\Phi w$, and do as many calculations as possible
in terms of $w$ instead of $x$ and $y$.  Fig.~\ref{fig:approx-uip}
summarizes the resulting algorithm.  (In fact it is not even necessary
to keep track of $w$ explicitly, but Fig.~\ref{fig:approx-uip} makes
$w$ explicit for clarity.)  For convenience we assume that $\Phi$ has
full column rank; if not, we can drop some columns from $\Phi$ and
adjust $U$ accordingly.

\begin{figure}
\rule[1ex]{\columnwidth}{0.5pt}
\begin{flushleft}
  In: $\Phi$, $U$, $q$; strictly feasible $x,y$; $\beta\in(0,1)$;
  $\epsilon>0$.\\
  Out: strictly feasible $x, y$ with $x\tr y\leq\epsilon$.
\end{flushleft}
\begin{centering}
  \begin{enumerate}
    \setlength{\itemsep}{.5ex}
    \setlength{\parskip}{0pt}
  \item Set $w \leftarrow \Phi^\dag(x-y+q)$.
  \item Stop if $x\tr y\leq\epsilon$.
  \item Compute $g$ and $r$ from~(\ref{eq:g-def})
    and~(\ref{eq:fast-r}).
  \item Compute $G$ and $h$ via~(\ref{eq:lo-d-G}--\ref{eq:lo-d-h}).
  \item Solve $G\Delta w = h$ for $\Delta w$.
  \item Solve $(X+Y) \Delta y = g - Y\Phi \Delta w$ for $\Delta y$.
  \item Compute $\Delta x = \Delta y + \Phi\Delta w$.
  \item Choose a step length $\theta\geq 0$ by~(\ref{eq:step-length}).
  \item Set $(x,y,w) \leftarrow (x,y,w) + \theta (\Delta x, \Delta y,
    \Delta w)$.
  \item Repeat from step 2.
\end{enumerate}
\rule[2ex]{\columnwidth}{0.5pt}
\vspace{-3ex}
\end{centering}
\caption{Potential reduction for projective LCPs.}
\label{fig:approx-uip}
\end{figure}

We are given a strictly feasible pair $x,y$ to start, so our initial
$w$ is just $\Phi^\dag(x-y+q)$; we then have $x-y+q=\Phi w$ by
Lemma~\ref{lem:inrange}.  We update $w$ by adjoining the equation
$\Delta x - \Delta y = \Phi \Delta w$ to the
system~(\ref{eq:update-direction}).  With this extra
constraint,~(\ref{eq:update-direction}) becomes:
\begin{equation}
\label{eq:approx-direction}
\left(
\begin{array}{ccc}
Y & X & 0\\
-M & I & 0\\
-I & I & \Phi
\end{array}
\right)\left(
\begin{array}{c}
\Delta x\\
\Delta y\\
\Delta w
\end{array}
\right) = \left(
\begin{array}{c}
g\\
r\\
0
\end{array}
\right)
\end{equation}
Note that the extra constraint does not change the sequence of points
$(x,y)$ visited by our potential reduction algorithm: its only
effect is to allow us to track $w$ and
solve~(\ref{eq:approx-direction}) efficiently.  

To solve~(\ref{eq:approx-direction}) efficiently, we will run several
steps of block Gaussian elimination analytically.  First use the last
block row of~(\ref{eq:approx-direction}) to eliminate the first block
column:
\begin{equation}
\label{eq:first-elim}
\left(
\begin{array}{ccc}
 X+Y & Y\Phi\\
 I-M & -M\Phi\\
\end{array}
\right)\left(
\begin{array}{c}
\Delta y\\
\Delta w
\end{array}
\right) = \left(
\begin{array}{c}
g
\\
r
\end{array}
\right)
\end{equation}
Then note that strict feasibility implies that $X+Y$ is nonsingular.
So, we can use the first block row of~(\ref{eq:first-elim}) to
eliminate the first block column:
\begin{equation}
\label{eq:second-elim}
\begin{array}{r}
[(M-I)(X+Y)^{-1} Y\Phi-M\Phi]
\Delta w
 = \qquad \\[.5ex]
(M-I)(X+Y)^{-1}g + r
\end{array}
\end{equation}
Finally, we can left-multiply~(\ref{eq:second-elim}) by $\Phi\tr$ to
reduce to
\begin{equation}
\label{eq:third-elim}
G \Delta w = h
\end{equation}
where
\begin{align}
G &= \Phi\tr [(M - I)(X+Y)^{-1} Y - M]\Phi \nonumber \\
&= (\Phi\tr \Phi U - \Phi\tr)(X+Y)^{-1} Y\Phi - \Phi\tr (\Phi U)\Phi \label{eq:lo-d-G} \\
h &= \Phi\tr [(M-I)(X+Y)^{-1}g + r] \nonumber \\
&=  (\Phi\tr\Phi U-\Phi\tr)(X+Y)^{-1}g + \Phi\tr r \label{eq:lo-d-h}
\end{align}
Eqs.~\ref{eq:lo-d-G}--\ref{eq:lo-d-h} show how to build $G$ and $h$ in
time $O(nk^2)$, starting from $x$, $y$, $r$, $g$, $\Phi\tr\Phi U$,
and $\Phi$, which together require $O(nk)$ storage.  The vectors $r$ and $g$
can be calculated efficiently using~(\ref{eq:g-def}) and
the representation $M = \Phi U + I - \Phi\Phi^\dag$: we compute 
\begin{align}
\label{eq:fast-r}
Mx = \Phi (Ux - \Phi^\dag x) + x
\end{align}
For the term $\Phi^\dag x$, it may help to precompute a factorization
such as the $QR$ decomposition of $\Phi$.

We can then
solve~(\ref{eq:third-elim}) for $\Delta w$ in time $O(k^3)$ or better.
(Lemma~\ref{lem:Ginv}, whose proof is in the appendix, ensures that
$\Delta w$ is uniquely determined.)
Since $X+Y$ is nonsingular, we can use the first block row
of~(\ref{eq:first-elim}) to solve for $\Delta y$ in time $O(nk)$.
Finally, we can use the last block row of~(\ref{eq:approx-direction})
to solve for $\Delta x$ in time $O(nk)$.

\begin{lemma}
  \label{lem:Ginv}
  If $x,y>0$, $\Phi$ has full column rank, and $\Phi U$ is positive
  semidefinite, then the matrix $G$ defined in~(\ref{eq:lo-d-G}) is
  invertible.
\end{lemma}

Since $k<n$, the total time per iteration is $O(nk^2)$, as claimed
earlier---potentially substantially faster than an iteration of the
potential reduction method on an arbitrary monotone LCP\@.  Since we
are performing the exact same sequence of updates to $x$ and $y$ as
the general potential-reduction algorithm running on $(M,q)$, our
bounds from Sec.~\ref{sec:proof} continue to hold: we take the same
number of iterations and reach the same final error level.  So, we
have proven:

\begin{theorem}
  Suppose $n\geq 2$, and take the parameter $\kappa$ from our
  potential~(\ref{eq:pot}) to be $\kappa=\sqrt n$.  Suppose $x$ and
  $y$ are strictly feasible for the monotone projective LCP $(M,q)$,
  where $M=\Phi U + \Pi_0$, $\Pi_0=I-\Phi\Phi^\dag$, and $\Phi$ has
  full column rank.  Then the algorithm of Fig.~\ref{fig:approx-uip},
  with $\beta = \frac{n}{n+\kappa}$, maintains strict feasibility and
  guarantees a potential reduction of at least $\frac{1}{5}$ per step.
\end{theorem}

\subsection*{Acknowledgements}

This work was supported by ONR MURI grant number N00014-09-1-1052.

\bibliographystyle{unsrt}
\bibliography{lcpmdp}

\appendix

\section{Proofs of Lemmas}

\begin{proofof}{Lemma \ref{lem:pot-bounds-gap}}
  Define $u_i=y_i x_i$ and let $u = \sum_i u_i$.  Then $x\tr y = u$, and
  \begin{align*}
    n &\ln x\tr y - \sum_{i=1}^n \ln x_i - \sum_{i=1}^n \ln y_i \\
    &= -\sum_{i=1}^n\ln u_i/u > 0
  \end{align*}
  since $0<u_i/u<1$ for all $i$.
\end{proofof}

\begin{proofof}{Lem.~\ref{lem:unique-step}}
  Since $x>0$, $X$ is invertible.  So, we can use Gaussian elimination
  on~(\ref{eq:update-direction}) to arrive at
  \[
  (-M-X^{-1}Y)\Delta x = Mx+q-y-X^{-1}(t\one-x\circ y)
  \]
  (In particular, subtract $X^{-1}$ times the first row from the
  second row).

  Since $M$ is positive semidefinite and $x,y>0$, $M+X^{-1}Y$ is
  strictly positive definite, and so we can solve uniquely for $\Delta
  x$.  We can then substitute $\Delta x$ into the first row
  of~(\ref{eq:update-direction}), which leads to a unique solution for
  $\Delta y$ since $X$ is invertible.
\end{proofof}

\begin{proofof}{Lemma~\ref{lem:ln-upper}}
  By construction, the left-hand and right-hand sides
  of~(\ref{eq:ln-upper}) match in value and first derivative at
  $\Delta z=0$.  The second derivative of the left-hand side with
  respect to $\Delta z$ is $(z+\Delta z)^{-2}$, while that of the
  right-hand side is $\tau z^{-2}$.  When $\Delta z \geq 0$, $\tau
  z^{-2}\geq(z+\Delta z)^{-2}$, so~(\ref{eq:ln-upper}) holds.  When
  $\Delta z < 0$,~(\ref{eq:ln-upper}) holds as long as
  \begin{align*}
    \int^0_{\Delta z} (z+\xi)^{-2}d\xi 
    &\leq \int^0_{\Delta z}\tau z^{-2}d\xi\\
    \bigl[-(z+\xi)^{-1}\bigr]^0_{\Delta z}
    &\leq \tau z^{-2}\bigl[\xi\bigr]^0_{\Delta z}\\
    \frac{1}{z+\Delta z}-\frac{1}{z} &\leq -\tau z^{-2}\Delta z\\
    -\Delta z &\leq \textstyle -\tau (1+\frac{\Delta z}{z}) \Delta z\\
    1 &\leq \textstyle \tau (1+\frac{\Delta z}{z})\\
    \textstyle \frac{1}{\tau}-1 &\leq \textstyle \frac{\Delta z}{z}
  \end{align*}
  (In the third line from the bottom we multiply through by
  $z(z+\Delta z)>0$, and in the second line from the bottom we divide
  through by $-\Delta z>0$.)
\end{proofof}

\begin{proofof}{Lemma~\ref{lem:glen}}
  Write $\xi = S^{-1}\one$, i.e., $\xi_i=1/s_i$.  Write
  $\mu=\frac{s\tr s}{n}$.  We have:
  \begin{align*}
    \|S^{-1}g\|^2 &= \textstyle \|\beta\mu\xi - s\|^2 \\
    &= \textstyle \|\beta(\mu\xi-s) - (1-\beta)s\|^2 \\
    &= \textstyle \|\beta(\mu\xi-s)\|^2 + \|(1-\beta)s\|^2
  \end{align*}
  since $(\mu\xi-s)\tr s = n\mu - s\tr s = 0$.  The first term is a
  sum of squares, so is at least as large as any of its components:
  \begin{align*}
    \|S^{-1}g\|^2 &\geq \textstyle \beta^2(\mu\frac{1}{s_0}-s_0)^2 + (1-\beta)^2s\tr s \\
    &= \textstyle \frac{\beta^2}{s_0^2}[(\mu-s_0^2)^2 + \frac{(1-\beta)^2}{\beta^2} s_0^2 n\mu] \\
    &= \textstyle \frac{\beta^2}{s_0^2}[\mu^2-\mu s_0^2 + s_0^4]
  \end{align*}
  since $1-\beta = \frac{\sqrt n}{n+\sqrt n}$, so
  $n\frac{(1-\beta)^2}{\beta^2} = n \frac{(\sqrt n)^2}{n^2} = 1$.
  Finally, we can complete the square of $(\frac{1}{2}\mu-s_0^2)$, getting:
  \begin{align*}
    \|S^{-1}g\|^2 &\geq \textstyle \frac{\beta^2}{s_0^2}[\frac{3}{4}\mu^2+ (\frac{1}{2}\mu- s_0^2)^2] \\
    &\geq \textstyle \frac{3}{4}\mu^2 \frac{\beta^2}{s_0^2}
  \end{align*}
  as desired.
\end{proofof}

\begin{proofof}{Lemma~\ref{lem:Ginv}}
  Let $D = (X+Y)^{-1}Y$.  Note that $D$ is diagonal, with all elements
  strictly between $0$ and $1$ (the $i$th diagonal element is $y_i/(x_i+y_i)$).
  Since $\Phi\tr \Pi_0=0$, we can rewrite the first line of~(\ref{eq:lo-d-G}) as:
  \begin{align*}
    G &= \Phi\tr [\Phi U(D-I) - D]\Phi
  \end{align*}
  The matrix $\Phi U(I-D)$ is positive semidefinite: it has the same
  eigenvalues as its similarity transform 
  \[ 
  (I-D)^{1/2}\Phi U(I-D)^{1/2} 
  \]
  which is positive semidefinite since is of the form $X\tr A X$ for
  real matrices $A$ and $X$ with $A=\Phi U$ positive semidefinite.
  The matrix $D$ is strictly positive definite, since it is diagonal
  with strictly positive diagonal elements.  So, the sum $\Phi
  U(I-D)+D$ is also strictly positive definite, as is $\Phi\tr(\Phi
  U(I-D)+D)\Phi=-G$ since $\Phi$ has full rank.  So, $G$ is invertible 
  as claimed.
\end{proofof}

\end{document}